\documentclass[a4paper]{svproc} 

\usepackage[T1]{fontenc}

\usepackage{graphicx}
\usepackage{wrapfig}

\usepackage{hyperref}
\usepackage{color}

\usepackage{mathtools}
\usepackage{amsfonts}
\newtheorem{assumption}{Assumption}
\usepackage{xfrac}
\usepackage{multirow}
\usepackage{algorithm}
\usepackage{algpseudocode}

\begin{document}
\setlength{\abovedisplayskip}{5pt}
\setlength{\belowdisplayskip}{5pt}
\emergencystretch=20pt

\title{Embodied Active Learning of Generative Sensor-Object Models}
\author{Allison Pinosky 
\and
Todd D. Murphey  
}
\authorrunning{A. Pinosky and T.D. Murphey}
\institute{Northwestern University, Evanston, IL 60208 \email{apinosky@u.northwestern.edu}}
\maketitle

\begin{abstract}
When a robot encounters a novel object, how should it respond---what data should it collect---so that it can find the object in the future? In this work, we present a method for learning image features of an unknown number of novel objects. To do this, we use active coverage with respect to latent uncertainties of the novel descriptions. We apply ergodic stability and PAC-Bayes theory to extend statistical guarantees for VAEs to embodied agents. We demonstrate the method in hardware with a robotic arm; the pipeline is also implemented in a simulated environment. Algorithms and simulation are available open source \url{https://sites.google.com/u.northwestern.edu/embodied-learning-hardware}.

\end{abstract}

\section{Introduction}

Deep generative learning has enabled large leaps in state-of-the-art performance on recognition tasks. These improvements have occurred alongside the widespread use of pre-existing datasets for training---most of which are based on images, videos, or text. 
For many interesting scenarios--such as performing tasks underwater or learning to use an ultrasound imager---there may be insufficient existing data to train an algorithm offline.
Embodied agents, like robots, could fill this gap by simultaneously collecting data and training models online. Furthermore, it would be useful for a robot to be able to generate new representations on-the-fly either when it encounters new objects or more broadly if it determines that its current representation is insufficient. 
In such cases, agents can either collect training data through random sampling or plan to actively collect future data samples using an information measure. 
For training data to be suitable for many generative methods, the data must contain
independent and identically distributed (\emph{i.i.d.}) samples with respect to some underlying true distribution. This assumption generally does not hold for real-world problems in which agents build their training sets sequentially. 
Particularly because embodied agents cannot move discontinuously, so trajectories are at least continuous and there is a real-world cost to movement.

\vspace{-1em}
\subsubsection{Our Approach}
In this work, we treat the robot as an \textit{active learner}, which has control over the data which is collected from the workspace. 
For examples in this work, the agent iteratively collects data and trains a variational autoencoder (VAE) model in real time. 
We restrict our robot to move in the planar workspace to collect data with an RGB camera (see Fig.\ref{fig:franka-real}), but the method is designed to scale to higher dimensions as well as other sensors. The robot has no pre-specified representation of camera properties or the objects in the workspace. 

Our approach is related to work on learning theory and active view selection.
As generative methods have grown in popularity, there have been an increasing number of papers that seek to understand the empirical success of these methods. Many works use the Probably Approximately Correct (PAC) learning framework as a baseline for establishing theoretical guarantees under \textit{i.i.d.} data assumptions ~\cite{foong2021tight,Shalizi2013,zhou2023toward,cherief2022pac,mbacke2024statistical}. 
While other related work focuses on establishing guarantees for non-\emph{i.i.d.} data \cite{gao2016learnability,mohri2007stability,zhang2023asymptotically}, we seek to collect data in a way that satisfies the data requirements of the learning algorithm.

Many object recognition tasks use \textit{next-best-view} to select the next view to evaluate during inference~\cite{chen2018veram,Wu2015shapenets,Johns2016pairwise}. This method could also be applied to the data collection process, but it is a greedy approach which risks converging to local minima and generating redundant samples. Instead, we use the active ergodic exploration method in \cite{abraham2020kle3,prabhakar2022mechanical} to optimally select the most informative next \emph{set of views}, while considering the current state of the robot, the previously collected data, and the current state of the learning representation.
We apply ergodic stability and PAC-Bayes Theory to extend statistical guarantees for VAEs to embodied agents.
The contributions of the presented work are 
\begin{enumerate}
\item a method for actively learning latent features of unknown objects, 
\item theoretical results on how ergodic stability of information states enables embodied learning, and
\item experimental results demonstrating the ability of our method to generate information-rich latent spaces through active learning.
\end{enumerate}

\setlength{\abovecaptionskip}{2pt plus 4pt minus 2pt}
\begin{figure}[tb]      
    \centering
    \includegraphics[width=0.9\columnwidth]{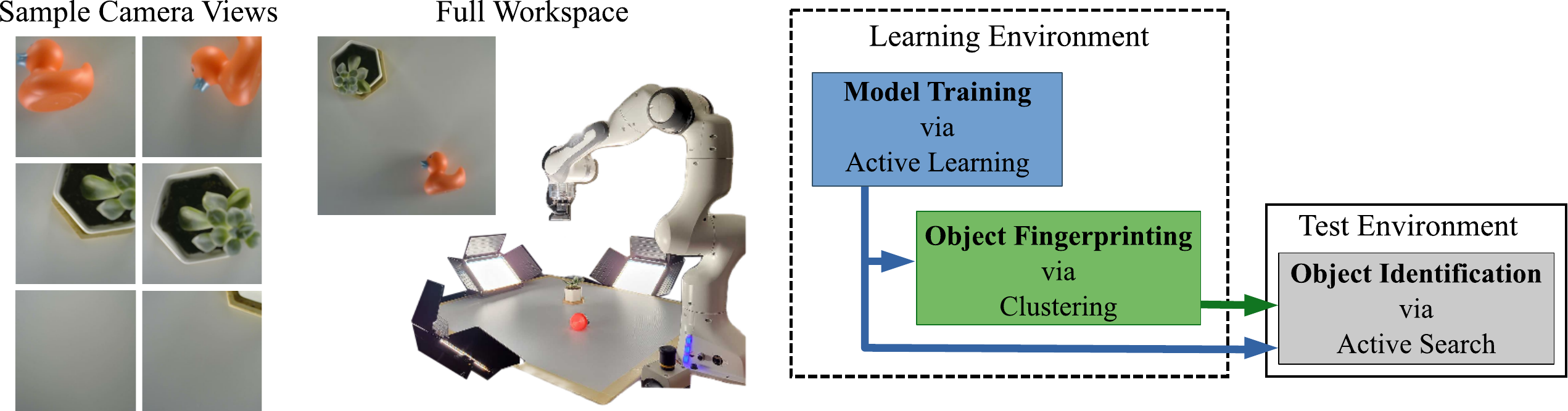} 
    \caption{
    \textbf{Test Environment.}
    Hardware experiments were performed on a robot arm with a webcam attached to the end-effector. The robot controls planar end-effector states $(x,y,\theta)$. 
    The goal is for the robot to explore a workspace while simultaneously building a representation of all observed objects in real-time. 
    It takes approximately $15$ minutes to run $3000$ exploration steps and $9000$ model updates. After learning, the model and learned objects can be used for future object identification tasks.
    }\label{fig:franka-real} 
    \vspace{-1.75em}
\end{figure}

The sections of the paper are organized as follows. In Section~\ref{sec:prelim}, we provide background on VAEs, PAC-Bayes theory, and ergodic control. In Section~\ref{sec:theory}, we present our method for extending these concepts to embodied agents with Conditional VAEs. In Sections \ref{sec:experiments} and \ref{sec:numerical-results}, we provide the methods and results for robot experiments which demonstrate the ability of our embodied agent to learn active latent spaces for object recognition. 

\section{Preliminaries}\label{sec:prelim}

In this section, we provide background necessary to prove how our method is able to learn in an embodied context.  
First, we introduce the variational autoencoder (VAE)~\cite{kingma2014adam,rezende2014stochastic}. Then, we review PAC-Bayes learning theory and prior work, which establishes statistical guarantees for VAEs~\cite{mbacke2024statistical}.  Finally, we present an active exploration method called ergodic control \cite{abraham2020kle3,miller2016ergodicexploration,mavrommati2018coverage}. 

\subsection{Variational Autoencoders}\label{sec:vae}
A VAE is a directed graphical model that learns a low-dimensional latent space in an unsupervised fashion by training two networks---an encoder and a decoder. The encoder compresses input data into a low-dimensional parameterized latent space distribution, and the decoder generates data predictions from the sampled latent space ~\cite{kingma2014auto,rezende2014stochastic}. 
VAEs are typically trained with the objective: 
\begin{equation}\label{eq:vae-loss}
    \mathcal{L}(\theta,\phi;y) 
    = \,\underbrace{ \mathbb{E}_{q_\phi(z|y)}[\log p_\theta(y|z)]}_\text{reconstruction loss} - \underbrace{\beta D_{KL}(q_\phi(z|y) \parallel p(z)) }_\text{latent space regularization} 
\end{equation}
where $y$ are sensor data (e.g. an image), $q_\phi(z|y)$ is the latent space distribution represented by encoder parameters $\phi$, $p_\theta(y|z)$ is the decoder network parameterized by $\theta$, $p(z)$ is the prior latent space distribution, $\beta$ is a hyperparameter, and $D_{KL}( \cdot \Vert \cdot)$ is the Kullback-Leibler (KL) Divergence. 

A key assumption of the VAE objective is that it is optimized over a complete dataset that contains \emph{i.i.d.} samples.\footnote{For mini-batch training, it is also assumed that data are drawn \emph{i.i.d.} from the dataset.}
Two learning behaviors often observed during VAE training are learning-lag and posterior collapse~\cite{he2018lagging,lucas2019understanding,Sonderby2016laddervae,yeung2017tackling}.
Learning lag refers to delay between learning the two loss terms, as the latent space learning often lags the reconstruction.
Posterior collapse describes instances where the latent space collapses to the prior, which is typically a normal distribution $\mathcal{N}(\mathbf{0},\mathbf{I})$.
When full posterior collapse occurs,
the VAE can learn a good generative model of the data (i.e. good decoder) but fail to learn good representations of the individual data points (i.e. poor encoded latent space). 
Prior work seeks to mitigate learning-lag and posterior collapse with batch normalization and $\beta$ annealing \cite{Ioffe2015batchorm,Bowman2016aneal,higgins2017betavae}. But parameter tuning is sensitive to the learning dataset. 

VAE networks can also be defined as functions, which we will use for the PAC-Bayes theory  presented in the next section \cite{mbacke2024statistical}. 
\begin{definition}[VAE Encoder and Decoder Networks]\label{def:vae-nets}
    Given a Euclidean instance space $\mathcal{Y}$ and a latent space $\mathcal{Z}=\mathbb{R}^d$, let $q_\phi(z|y)$ be a Gaussian latent space distribution $\mathcal{N}\left(\mu_\phi(y),\text{diag}\left(\sigma^2_\phi(y)\right)\right)$, where $\mu_\phi(y) : \mathcal{Y}\!\rightarrow\!\mathbb{R}^d$ and $\sigma_\phi(y) : \mathcal{Y}\!\rightarrow\!\mathbb{R}^d_{\geq 0}$.
    We can define the encoder and decoder as functions, such that 
    \begin{align}
        Q_\phi(y) & : \mathcal{Y}\!\rightarrow\!\mathbb{R}^{2d},  \text{ where } Q_\phi(y) =\begin{bmatrix} \mu_\phi(y) \\ \sigma_\phi(y)\end{bmatrix},         
        \\
        g_\theta(z) & : \mathcal{Z}\!\rightarrow\!\mathcal{Y}, \text{ where latent space samples }  z\sim q_\phi(z|y).
    \end{align}
\end{definition}
With this definition, we are ready to introduce PAC-Bayes theory.

\subsection{PAC-Bayes Theory}\label{sec:pac}

A class of functions is PAC-learnable if an algorithm can produce a function that recreates the input-output mapping of an arbitrary target function with high probability (at least $1-\delta$) and low error (at most $\epsilon$). PAC-Bayes extends PAC learning theory to probability distributions over a class of hypotheses and provides an upper bound on the model's empirical risk and its population class~\cite{catoni2003pac,mcallester1998some}. 
A further extension of PAC-Bayes to conditional priors for variational autoencoders was presented in~\cite{mbacke2024statistical}. In the rest of this section, we summarize the results from \cite{mbacke2024statistical}. 
The main theorem relies on the following two assumptions.
\begin{assumption}[Posterior and Loss Function Properties]\label{asmpt:functions}
    A distribution $y\mapsto q(\cdot|y)$ and a loss function $\ell$ satisfy Assumption~\ref{asmpt:functions} with constant $K>0$ if there exists a family $\mathcal{E}$ of functions $\mathcal{H}\!\rightarrow\!\mathbb{R}$ such that the following properties hold  
    \begin{enumerate}
        \item The function $y\mapsto q(\cdot|y)$ is continuous in the sense that for any $y_1,y_2 \in \mathcal{Y}$, $$d_\mathcal{E}(q(h|y_1),q(h|y_2)) \leq K d(y_1,y_2),$$ where $d_\mathcal{E}$ are Integral Probability Metrics (IMP, see \cite{muller1997integral}) defined on the family of functions $\mathcal{E}$ and $d$ is an underlying metric on $\mathcal{Y}$.
        \item The function $\ell(\cdot,y):\mathcal{H}\!\rightarrow\!\mathbb{R}$ is in $\mathcal{E}$ for any $y\in\mathcal{Y}$.
    \end{enumerate}
\end{assumption}
\begin{assumption}[VAE Encoder and Decoder]\label{asmpt:lipshitz-VAE}
    The encoder and decoder are Lipschitz-continuous w.r.t. their inputs meaning there exist real numbers \newline ${K_{\phi}, K_\theta > 0}$ such that for any $y_1, y_2 \in \mathcal{Y}$ and $z_1, z_2 \in \mathcal{Z}$,
    \begin{align}
        \Vert Q_\phi (y_1) - Q_\phi (y_2)\Vert & \leq K_{\phi} \Vert y_1 - y_2 \Vert 
        \\
        \Vert g_\theta (z_1) - g_\theta (z_2) \Vert & \leq K_\theta \Vert z_1 - z_2 \Vert ,
    \end{align}
    where $\Vert\cdot\Vert$ denotes the $L_2$ norm. 
\end{assumption}
We define the loss function as
\begin{align}
    \ell^\theta : \mathcal{Z} \times \mathcal {Y} \rightarrow [0,\infty) \text{ where } \ell^\theta(z,y) = \Vert y - g_\theta(z)\Vert.
\end{align}
Assumption~\ref{asmpt:functions} is formulated generally for conditional distributions.  We require a further statement to show that for the VAE loss function $\ell^\theta(z,y)$ there is a family $\mathcal{E}$ for which the continuity assumption is satisfied with constant $K$.
\begin{lemma}[VAE Satisfies Assumption~\ref{asmpt:functions}]\label{lemma:vae}
    For a VAE with parameters $\phi$ and $\theta$ and let $K_\phi,K_\theta \in \mathbb{R}$ be the Lipschitz norms of the encoder and decoder respectively. Then the variational distribution $q_\phi(z|y)$  satisfies Assumption~\ref{asmpt:functions} with $\mathcal{E} = \{f:\mathcal{Z} \rightarrow\mathbb{R} \text{ s.t. } \Vert f \Vert_\text{Lip} \leq K_\theta\}$,  $ \ell = \ell^\theta(z,y)$, and $K =  K_\phi K_\theta $.
\end{lemma}
With these assumptions, we can introduce a PAC-Bayes bound for the VAE reconstruction loss. 
\begin{theorem}[VAE PAC-Bayes Bounds]\label{thm:PAC-VAE}
    Let $\mathcal{Y}$ be the instance space, $\zeta \in \mathcal{M}^1_+(\mathcal{Y})$ the data-generating distribution, $Z$ the latent space, $p(z) \in \mathcal{M}^1_+(\mathcal{Z})$ the prior distribution on the latent space, $\theta$ the decoder parameters, and $\delta \in (0,1)$, $\lambda > 0$ be real numbers. With probability at least $1-\delta$ over $S \stackrel{\text{\emph{i.i.d.}}}{\sim} \zeta$, the following holds for any posterior $ q_\phi (z | y )$: 
    \begin{align}
        \begin{split}
        &\mathop{\mathbb{E}}_{y\sim\zeta}\mathop{\mathbb{E}}_{z \sim q_\phi(z|y)} \ell^\theta(z,y) \!\leq\!
        \frac{1}{n}\sum^n_{i=1}  \mathop{\mathbb{E}}_{z \sim q_\phi(z|y_i)} \ell^\theta(z,y_i)  
        \!+\! \frac{1}{\lambda} \Bigg[ 
        \sum^n_{i=1} D_{KL}( q_\phi(z|y_i) \Vert p(z)) 
        \\ &
        \!+\!  \log{\frac{1}{\delta}} 
        \!+\! \frac{\lambda K_\phi K_\theta}{n} \sum^n_{i=1} \mathop{\mathbb{E}}_{y\sim\zeta} d(y,y_i) 
        \!+\! n \log \mathop{\mathbb{E}}_{z \sim p(z)}\mathop{\mathbb{E}}_{y\sim\zeta} e^{\frac{\lambda}{n}(\mathop{\mathbb{E}}_{y'\sim\zeta}\ell^\theta(z,y')-\ell^\theta(z,y))}
         \Bigg],
        \end{split}
    \end{align}
    where $K_\phi,K_\theta$ are  encoder and decoder Lipschitz norms and $d(y,y'){=}\Vert y-y'\Vert$.
\end{theorem}  
Similar to the VAE loss function, these bounds assume \emph{i.i.d.} data. 

In our embodied problem, the robot must collect data to train its learning representation. 
A naive approach to data collection could use random sampling of the test environment, but if the environment is sparse, many samples will be required to fully cover the space and many of the samples may be redundant or uninformative. Furthermore, there is no guarantee that data which is \emph{i.i.d.} with respect to workspace states would also be \emph{i.i.d.} with respect to the sensor data. Next, we introduce the control method which enables embodied data collection.

\subsection{Ergodic Control}\label{sec:control} 

In robotics, ergodic control treats the problem of exploration and exploitation as a problem of matching a target spatial distribution to a time-averaged distribution. 
As in~\cite{abraham2020kle3,miller2016ergodicexploration,mavrommati2018coverage}, one can calculate a controller that optimizes the ergodic metric such that the trajectory of the robot is ergodic with respect to a target distribution. 
In this work, we use a receding-horizon controller based on minimizing a KL-ergodic measure detailed in~\cite{abraham2020kle3} to synthesize future controls. 
The controller optimization minimizes the KL-divergence $D_{KL}(h \Vert d)$,  where $h$ is a target spatial distribution and  $d$ is the time-average statistics of the robot.
The time-average statistics $d$ of the robot are defined as
\begin{align}
    d(s | x(t)) &=\frac{1}{{T}_{r}}\int_{{t}_{i}-{T}_{r}}^{{t}_{i}+T} \frac{1}{\eta} \psi (s | x(t))dt \label{eq:time_average_statistics}
\end{align}
where $\psi(s|x(t))=\exp(-\frac{1}{2} \lVert s - x(t) \rVert^{2}_{\Sigma^{-1}}) $, $\eta$ is a normalization factor, $s$ are states drawn from the reachable workspace, $T$ is the planning horizon, and  $\Sigma$ specifies the width of the state coverage Gaussian. 

We will discuss our choice of $h$ in Sec.~\ref{sec:active-learning}. For now, the key idea is that if we specify the target distribution such that it represents the distribution of sensor data in the test environment, ergodic control can be used to collect data to match the target distribution.
Now, we are ready to present our extension of VAE bounds from Theorem \ref{thm:PAC-VAE} to embodied learning of conditional networks.

\section{Embodied CVAE Learning}\label{sec:theory}

Our embodied problem necessitates a connection between the physical states of the robot and the sensor data---the VAE requires \textit{i.i.d} sensor data, and the robot requires states to travel to to collect the sensor data. To connect these elements, we use a VAE variant called a conditional VAE (CVAE) as our learning representation. 
CVAEs augment the general VAE network architecture by adding conditional variables as inputs to both the encoder and decoder networks \cite{prabhakar2022mechanical,sohn2015learning}. 
CVAEs are typically trained with the objective function: 
\begin{equation}\label{eq:cvae-loss}
    \mathcal{L}(\theta,\phi;x,y) 
    = \, \underbrace{ \mathbb{E}_{q_\phi(z|x,y)}[\log p_\theta(y|z,x)]}_\text{reconstruction loss} - \underbrace{\beta D_{KL}(q_\phi(z|x,y) \parallel p(z)) }_\text{latent space regularization} 
\end{equation}
where $y$ are sensor data (e.g. an image), $x$ are conditional variables (e.g. robot state), $q_\phi(z|x,y)$ is the latent space distribution represented by encoder parameters $\phi$, and $p_\theta(y|z,x)$ is the decoder parameterized by $\theta$~\cite{kingma2014auto}. 

Similar to VAEs, CVAEs are susceptible to learning-lag and posterior collapse. The addition of conditional variables to the CVAE decoder means full posterior collapse can result in the decoder network 
learning a direct mapping from conditional variables to predicted sensor outputs.
The benefit of using a VAE is that the latent space models the correlations between samples. If the latent space fully collapses, we lose access to these correlations. In this work, the goal is to learn information-rich latent spaces, which could be used to solve future tasks like feature clustering and object recognition \cite{ranzato2007unsupervised,higgins2017betavae}. 
Therefore, considering the quality of the data collection process is necessary to prevent catastrophic posterior collapse, particularly in applications with real-time data collection. 

To generate PAC-Bayes bounds, we need to define our CVAE networks as functions similar to Def.~\ref{def:vae-nets}. The primary change from the previous definition is the addition of conditional state as an input to each function.
\begin{definition}[CVAE Encoder and Decoder Networks]\label{def:cvae-nets}
    Given Euclidean instance spaces  $\mathcal{X}$ and  $\mathcal{Y}$ and latent space $\mathcal{Z}=\mathbb{R}^d$, let $q_\phi(z|x,y)$ be a Gaussian latent space distribution $\mathcal{N}(\mu_\phi(x,y),\text{diag}(\sigma^2_\phi(x,y)))$, where 
    $\mu_\phi(x,y) :  \mathcal{X} \times \mathcal{Y}\!\rightarrow\!\mathbb{R}^d$ and $\sigma_\phi(x,y) : \mathcal{X} \times \mathcal{Y}\!\rightarrow\!\mathbb{R}^d_{\geq 0}$.
    We can define the encoder and decoder as functions:  
    \begin{align}
        Q_\phi(x,y) &: \mathcal{X} \times \mathcal{Y}\!\rightarrow\!\mathbb{R}^{2d},  \text{ where } Q_\phi(x,y) =\begin{bmatrix} \mu_\phi(x,y) \\ \sigma_\phi(x,y)\end{bmatrix},
        \\
        g_\theta(z,x) &  : \mathcal{X} \times \mathcal{Z}\!\rightarrow\!\mathcal{Y},  \text{ where latent space samples } z\sim q_\phi(z|x,y).
    \end{align}
\end{definition}
We must also modify Assumption~\ref{asmpt:functions} to include conditional states. 
\begin{assumption}[CVAE Encoder and Decoder]\label{asmpt:lipshitz-CVAE}
    The encoder and decoder are Lipschitz-continuous w.r.t. their inputs meaning there exist real numbers \newline ${K_\phi,K_\theta > 0}$ such that for any $x_1,x_2 \in \mathcal{X}$, $y_1,y_2 \in \mathcal{Y}$, and $z_1,z_2 \in \mathcal{Z}$,
    \begin{align}
        \Vert Q_\phi (x_1,y_1) - Q_\phi (x_2,y_2) \Vert & \leq K_{\phi} \left( \Vert x_1 - x_2 \Vert + \Vert y_1 - y_2 \Vert \right), \label{eq:q-lip} 
        \\
        \Vert g_\theta (z_1,x_1) - g_\theta (z_2,x_2) \Vert & \leq K_\theta \left( \Vert z_1 - z_2 \Vert + \Vert x_1 - x_2 \Vert \right).
    \end{align}
\end{assumption}
We define the CVAE loss function as
\begin{align}
    \ell^\theta : \mathcal{Z} \times \mathcal{X} \times \mathcal {Y} \rightarrow [0,\infty) \text{ where } \ell^\theta(z,x,y) = \Vert y - g_\theta(z,x)\Vert.
\end{align}
An additional statement is required to show that for $\ell^\theta(x,y,z)$ there is a family $\mathcal{E}$ for which the continuity assumption is satisfied with constant $K$.
The proof relies on the Wasserstein distance, which we define prior to the lemma.
\begin{definition}[Wasserstein distance~\cite{givens1984class}]
    Given empirical distributions $P$ and $Q$ with samples $x_1 \dots x_n$ and $y_1 \dots y_n$ respectively, the distance measure is
    \begin{equation}
        W_p(P,Q) = \left(\frac{1}{n}\sum\nolimits^n_{i=1} \Vert x_i - y_i \Vert^p\right)^{1/p}.
    \end{equation} 
    Given normal distributions $p=\mathcal{N}(\mu_p,\text{diag}(\sigma_p^2))$ and $q=\mathcal{N}(\mu_q,\text{diag}(\sigma_q^2))$, the distance measure is  
    \begin{equation}
        W_2(p,q)^2= \Vert \mu_p - \mu_q \Vert^2 + \Vert \sigma_p - \sigma_q \Vert^2.
    \end{equation} 
\end{definition}
\begin{lemma}[CVAE Satisfies Assumption \ref{asmpt:functions}]\label{lemma:cvae}
    For a CVAE with parameters $\phi$ and $\theta$, let $K_\phi,K_\theta {\,\in\,} \mathbb{R}$ be the Lipschitz norms of the encoder and decoder respectively. Then the variational distribution $q_\phi(z|x,y)$  satisfies Assumption \ref{asmpt:functions} with
    \begin{align}
        d_\mathcal{E}(q_\phi(z| x_1, y_1), q_\phi(z | x_2, y_2)) & \leq K_\phi K_\theta \left(\Vert x_1 \!-\! x_2 \Vert{+}\Vert y_1 \!-\! y_2 \Vert\right) {+} K_\theta \Vert x_1 \!-\! x_2 \Vert
        \\
    \text{ and }
        \ell^\theta(z,x,y) & \in \mathcal{E} \text{ for } \text{any } x\in\mathcal{X} \text{ and } y\in\mathcal{Y},
    \end{align}
    where $\mathcal{E} = \{f:\mathcal{Z} \times \mathcal{X} \rightarrow\mathbb{R} \text{ s.t. } \Vert f \Vert_\text{Lip} \leq K_\theta\}$.
\end{lemma}
\begin{proof}\label{proof:cvae_lemma} 
In Def.~\ref{def:cvae-nets}, we defined $q_\phi(z|x,y)$ as a normal distribution and $Q_\phi(x,y) =\begin{bmatrix} \mu_\phi(x,y), \sigma_\phi(x,y)\end{bmatrix}^T$. Since the Wasserstein-2 distance for normal distributions has a closed form solution and $Q$ is Lipschitz continuous as defined in Eq.~\ref{eq:q-lip}, we can define the following inequality
\begin{equation}\label{eq:w2_ineq}
     W_2(q_\phi(z|x_1,y_1) ,q_\phi(z|x_2,y_2) ) \leq K_{\phi} \left( \Vert x_1 - x_2 \Vert + \Vert y_1 - y_2 \Vert \right) .
\end{equation}
Using the definition of the family of functions $\mathcal{E} = \{f:\mathcal{Z} \times \mathcal{X} \rightarrow\mathbb{R} \text{ s.t. } \Vert f \Vert_\text{Lip} \leq K_\theta\}$ and the Kantorovich duality, we obtain the IMP $d_\mathcal{E}$ as
{
\medmuskip=0mu
\thinmuskip=0mu
\thickmuskip=0mu
\begin{equation}\label{eq:w1_ineq}
    d_\mathcal{E}(q_\phi(z| x_1, y_1), q_\phi(z | x_2, y_2)) {=} K_\theta (W_1(q_\phi(z|x_1,y_1) ,q_\phi(z|x_2,y_2)) {+} \Vert x_1 {-} x_2 \Vert).
\end{equation}
}
Since $W_1 \leq W_2$, combining Eq.~\ref{eq:w1_ineq}  and  Eq.~\ref{eq:w2_ineq} yields %  the following inequality
\begin{equation}
    d_\mathcal{E}(q_\phi(z| x_1, y_1), q_\phi(z | x_2, y_2)) \!\leq\! K_\phi K_\theta \left(\Vert x_1 \!-\! x_2 \Vert+\Vert y_1 \!-\! y_2 \Vert\right) + K_\theta \Vert x_1 \!-\! x_2 \Vert.
\end{equation}
Now we can examine the second part of Lemma~\ref{lemma:cvae}. Let $\ell = \ell^\theta $, $x_1,x_2 \in \mathcal{X}$, $y \in \mathcal{Y}$, and $z_1,z_2 \in \mathcal{Z}$. We have
\begin{align}
    \ell(z_1,x_1,&y) -  \ell(z_2,x_2,y) 
    = \Vert y - g_\theta(z_1,x_1) \Vert - \Vert y - g_\theta(z_2,x_2) \Vert \\
    &= \Vert y - g_\theta(z_1,x_1) + g_\theta(z_2,x_2)- g_\theta(z_2,x_2)\Vert - \Vert y - g_\theta(z_2,x_2) \Vert \\
    &\leq \Vert y- g_\theta(z_2,x_2) \Vert + \Vert  g_\theta(z,x) - g_\theta(z_1,x_1) \Vert - \Vert y - g_\theta(z_2,x_2) \Vert \\
    &= \Vert  g_\theta(z_2,x_2) - g_\theta(z_1,x_1) \Vert \\
    &\leq K_\theta \left( \Vert  z_2 - z_1 \Vert + \Vert  x_2 - x_1 \Vert \right)
\end{align}
where the first inequality uses the triangle inequality, and the second uses the Lipschitz assumption on $g_\theta$.
\qed
\end{proof}
With the first part of the lemma, we have established that the decoder is continuous with respect to the IMP $d_\mathcal{E}$ and the underlying metrics $\mathcal{X}$ and $\mathcal{Y}$. We have also shown that the loss function captures the continuity of the decoder. We can now prove a PAC-Bayes bound for the CVAE reconstruction loss.
\begin{theorem}[CVAE PAC-Bayes Bounds]\label{thm:PAC-CVAE}
    Let $\mathcal{X}$ and $\mathcal{Y}$ be the instance spaces, $\pi \in \mathcal{M}^1_+(\mathcal{X})$  and $\zeta \in \mathcal{M}^1_+(\mathcal{Y})$ the data-generating distributions, $Z$ the latent space, $p(z) \in \mathcal{M}^1_+(\mathcal{Z})$ the prior distribution, $\theta$ the decoder parameters, and $\delta \in (0,1)$, $\lambda > 0$ be real numbers. With probability at least $1-\delta$ over $S_x \stackrel{\text{\emph{i.i.d.}}}{\sim} \pi$~and~$S_y \stackrel{\text{\emph{i.i.d.}}}{\sim} \zeta$, the following holds for any posterior $ q_\phi (z | x, y )$: 
    \begin{align}
        & {\mathop{\mathbb{E}^{}}_{x\sim\pi}} \mathop{\mathbb{E}^{}}_{y\sim\zeta}\mathop{\; \mathbb{E} \; \ell^\theta(z,x,y)}_{ z \sim q_\phi(z|x,y) \hfill } \leq 
        \frac{1}{n}\sum^n_{i=1}  \mathop{\; \mathbb{E} \; \ell^\theta(z,x_i,y_i)}_{z \sim q_\phi(z|x_i,y_i) \hfill}   
        + \frac{1}{\lambda} \Bigg[  
        \sum^n_{i=1} D_{KL}( q_\phi(z|x_i,y_i) \Vert p(z)) \nonumber \\ &
        +  \log{\frac{1}{\delta}} 
        + \frac{\lambda K_\phi K_\theta}{n} \sum^n_{i=1} \left(\, \boxed{\mathop{\mathbb{E}}_{x\sim\pi} d(x,x_i)} + \mathop{\mathbb{E}}_{y\sim\zeta} d(y,y_i) \right)
        + \boxed{ \frac{\lambda K_\theta}{n} \sum^n_{i=1} \mathop{\mathbb{E}}_{x\sim\pi} d(x,x_i) }
        \nonumber  \\ &
        + n \log \mathop{\mathbb{E}}_{z \sim p(z)}{\mathop{\mathbb{E}}_{x\sim\pi} }\mathop{\mathbb{E}}_{y\sim\zeta} e^{{\textstyle \frac{\lambda}{n}}\big({\underset{x'\sim\pi}{\mathbb{E}}}\;\underset{y'\sim\zeta}{\mathbb{E}} \;\ell^\theta(z,x',y')-\ell^\theta(z,x,y)\big)}
         \Bigg],
    \end{align}
    where $K_\phi,K_\theta$ are  encoder and decoder Lipschitz norms and $d(y,y'){=}\Vert y-y'\Vert$.
\end{theorem}  
The major difference between this theorem and Theorem~\ref{thm:PAC-VAE} is that the inequality includes new terms due the decoder being conditioned on $x$ (boxed above). The second term necessitates adding an additional term to the CVAE loss function.  

\subsubsection{\normalfont \textit{Proof Idea.}} \vspace{-1em}
The proof of Theorem~\ref{thm:PAC-CVAE} follows the proof of Theorem 3.1 in \cite{mbacke2024statistical}, which uses standard PAC-Bayes techniques with one main difference. The proof starts with $n$ \emph{i.i.d.} samples from the prior $p(z)$, which allows us to apply the Donsker-Varadhan change of measure theorem to $n$ posteriors $q(z |x_i,y_i)$. Then, we show that the exponential moment with $n$ latent spaces instead of one is equal to the exponential moment obtained with one latent space. 

Similar to the VAE objective, a crucial assumption underlying the concept of PAC-Bayes bounds is the ability to draw \emph{i.i.d.} samples. Ideally, an agent (or algorithm) would exhaustively sample all regions of the data distribution simultaneously and produce a function from the ensemble of data samples. This is equivalent to offline learning, where the learning process has access to a complete dataset before training. When we instead consider an embodied learning process, an agent must navigate the exploration domain to gather samples in a time-ordered sequential sampling process. In general, such a sampling process does not produce \emph{i.i.d.} data~\cite{Dean2020}. However, ergodicity can provide a method of producing \emph{i.i.d.} data with Birkhoff's ergodic theorem~\cite{Moore2015}, which we restate below:

\begin{figure}[tb]      
    \centering
    \includegraphics[width=0.75\columnwidth]{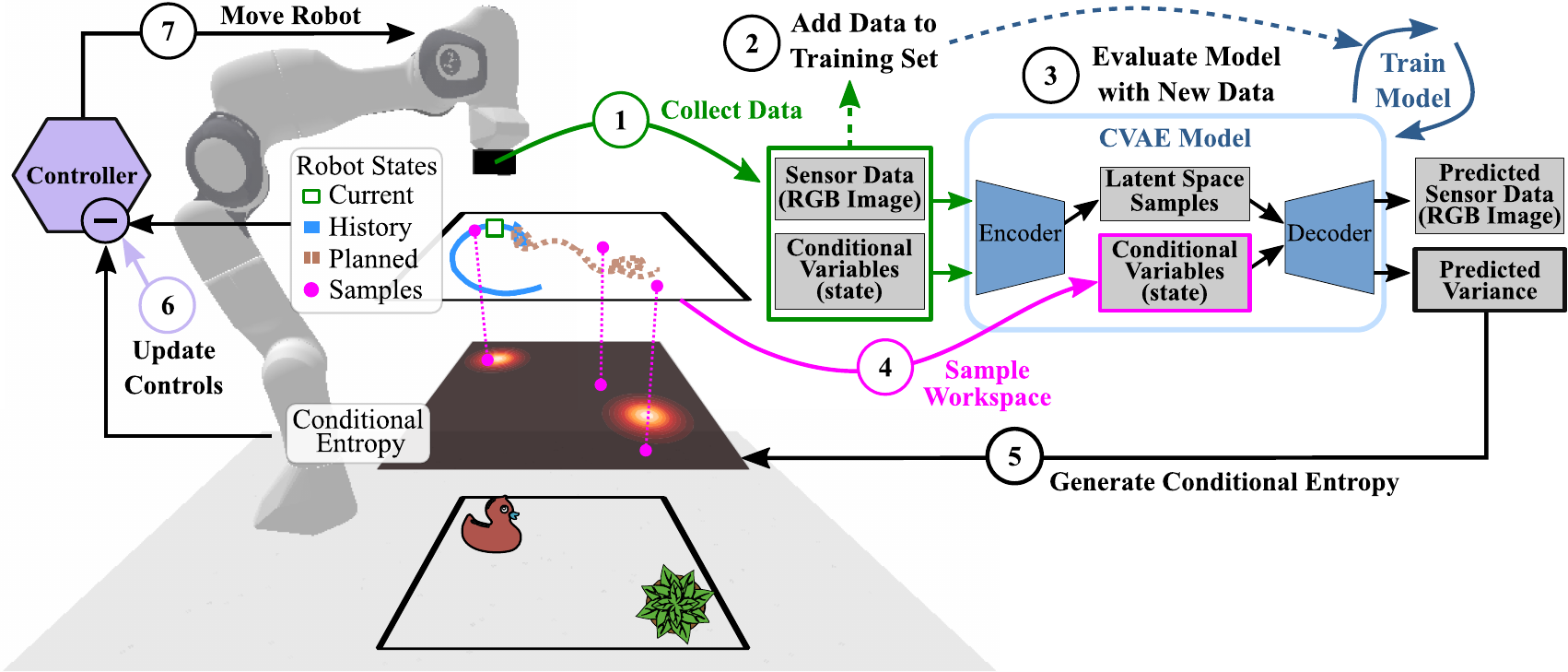}
    \caption{
    \textbf{Active Learning Process}
    Numerals 1-7 describe the steps completed during each active learning step.
    The layers below the robot arm show, (I)~previously visited states, future planned states, and samples drawn from the reachable workspace (see Eq.~\ref{eq:time_average_statistics}), (II)~conditional entropy of the model over the workspace samples (see Eq.~\ref{eq:cond-ent}), and (III)~objects in the workspace. 
    Model training occurs in parallel to data collection, so the robot trajectory, model, dataset, and conditional entropy all evolve continually.
    }\label{fig:learning} 
    \vspace{-1em}
\end{figure}

\begin{theorem}[Birkhoff's Ergodic Theorem]\label{thm:birkhoff}
    Let $\{x_t\}_{t\in N}$ be an aperiodic and irreducible Markov process on a state space $X$ with invariant measure $\rho$ and let $f : \mathcal{X}\!\rightarrow\!\mathbb{R}$ be any measurable function with $\mathbb{E}[|f (x)|] < \infty$. Then, 
    \begin{equation}
        \lim_{N\rightarrow\infty} \frac{1}{N}\sum\nolimits_{t=1}^{N}f(x_t) = \mathbb{E}_{x_0\sim\rho}[f(x_0)]     
    \end{equation}
almost surely.
\end{theorem}
Informally, theorem states that the time average of any function of an ergodic Markov chain is  equal to its ensemble average.
We can use ergodic control strategies from ~\cite{mavrommati2018coverage,abraham2020kle3} to generate data with \emph{i.i.d.} properties asymptotically.

\section{Methods}\label{sec:experiments}
With embodied learning, how informative the model is depends on what data the robot has collected. 
Previous work by \cite{prabhakar2022mechanical} showed ergodic control can be used to efficiently collect data for learning predictive sensor models in simulation; models were evaluated based on sensor reconstruction  quality. In this work, we aim to learn information-rich latent spaces to enable the models to be used for tasks like object identification.
Our experimental pipeline contains three parts, which are shown on the right side of Fig.~\ref{fig:franka-real}---active learning, object fingerprinting, and object identification. The \hyperlink{appendix}{Appendix} contains additional implementation details.

\subsubsection{Active Learning}\label{sec:active-learning} \vspace{-1em}
Our active learning process couples the robot data collection process (exploration) with the CVAE model learning process. 
Alg.~\ref{alg:learning} and the numeral bubbles in Fig.~\ref{fig:learning} describe the process completed during each robot exploration step---1)~collect data, 2)~add data to training set, 3)~evaluate the CVAE model for latest collected data, 4)~sample workspace, 5)~generate conditional entropy distribution using predicted decoder variance evaluated at each workspace sample, 6)~update controller to minimize the difference between visited states and conditional entropy, and 7)~move robot to new state. 
This cycle is repeated until learning is complete. Model training occurs in parallel to data collection, so the robot trajectory, model, and conditional entropy are all continually evolving.

\subsubsection{Control Formulation} \vspace{-1em}
We use the method from \cite{prabhakar2022mechanical}, which incorporates model uncertainty into the ergodic control strategy by specifying the target distribution as the conditional entropy distribution of model uncertainty over the possible search space. We define the CVAE conditional entropy distribution as
\begin{equation}\label{eq:cond-ent} 
    h(s|x,y)=\exp {\left(\mathbb{H}(p_\theta (y|z,s))\right)}^k 
\end{equation}
where $\mathbb{H}$ is the entropy of the decoder network $p_\theta$ given a latent space conditioned on the current robot state and sensor data $z\sim{q_\phi(z|x,y)}$ and $k$ is an effective temperature parameter that exponentially weighs regions of high importance. 

{\centering
\begin{minipage}{\linewidth}
\begin{algorithm}[H]
\fontsize{8}{10}\selectfont % footnotesize overwritten in class
\caption{\small Active Learning Process}
\label{alg:learning}
\begin{algorithmic}[1]
\State Randomly initialize CVAE encoder $q_\phi(z|x,y)$ and decoder $p_\theta(y|z,x)$. Initialize memory buffer $\mathcal{D}$, prediction horizon $H$, state coverage $\Sigma$, and latent space regularization scale $\chi$. 
\While{task not done}
\State Collect current state $x_c$ and sensor data $y_c$ and add to buffer $ \mathcal{D} \gets \mathcal{D}  + \{ x_c, y_c \} $ 
\State Generate latent space samples for current data $z_c \sim  q_\phi \left( z | x_c, y_c \right)$
\State Draw $S$ states from reachable workspace 
\State Draw $N$ previous states from buffer and predict future states over horizon $H$
\State Update distributions over workspace samples $s$
\State \quad $ k \leftarrow \frac{1}{S}\sum_{i=1}^S\max\limits_{x \in \mathcal{D}} \left( \exp \left(-\frac{1}{2}{\parallel}s_i-x{\parallel}_{{{{\Sigma }}}^{-1}}^{2}\right)\right)$   \Comment{Workspace Coverage}
\State \quad $h(s) \leftarrow 
   \exp {\left(\mathbb{H}(p_\theta\left(y|z_c , s \right)\right)}^k $ \Comment{Conditional Entropy Distribution}
\State \quad $ d(s) \leftarrow \sum_{t=1}^{N+H} \exp \left(-\frac{1}{2}{\parallel}s-x_t{\parallel }_{{{{\Sigma }}}^{-1}}^{2}\right)$ \Comment{Trajectory Distribution}
\State Update controller by minimizing $ D_{KL}(h(s) \parallel d(s))$ 
\State Update CVAE with mini-batches from $\mathcal{D}$ \hfill $\star$ Can be asynchronous
\State \quad $ \begin{matrix} \log_{10}\beta & , & \gamma \end{matrix} \leftarrow \begin{matrix}
    -\log_{10}\left(\min \left({h}/{\max{h}}\right) \right) - \chi & , &  0.1 k
\end{matrix}$ \Comment{Hyperparameters}  \label{line:hyperparams}
\State  \quad 
$ {\phi,\theta} \leftarrow $ Update CVAE parameters with gradients from $\mathcal{L}$ in Eq.~\ref{eq:cvae-loss-new}
\State Apply next action from controller to move robot
\EndWhile
\end{algorithmic}
\end{algorithm}
\end{minipage}
\vspace{0.5em}
}

One way to interpret the conditional entropy distribution is as the expected information density. Areas with high conditional entropy are \emph{very interesting} and areas of low conditional entropy are \emph{not very interesting}. 
By collecting data from the environment which matches the conditional entropy distribution, we  satisfy the \emph{i.i.d.} requirements of the CVAE learning bounds. 
We use the KL-ergodic measure $D_{KL}(h \Vert d)$ to assess the data collection quality over time.

\subsubsection{Embodied CVAE Loss Formulation}\label{sec:cvae-loss} \vspace{-1em}
To satisfy CVAE PAC-Bayes Bounds from Theorem~\ref{thm:PAC-CVAE}, we augment the CVAE loss function from Eq.~\ref{eq:cvae-loss} with an additional conditional term such that
\begin{equation}\label{eq:cvae-loss-new}
\begin{split}
    \mathcal{L}(\theta,\phi;x,y)
    = \, &\underbrace{ \mathbb{E}_{q_\phi(z|x,y)}[\log p_\theta(y|z,x)]}_\text{end-to-end reconstruction loss} - \beta \underbrace{ D_{KL}(q_\phi(z|x,y) \parallel p(z)) }_\text{latent space regularization} \\
    & + \gamma \underbrace{ \mathbb{E}_{q_\phi(z|x,y)}[\log p_\theta(\hat{y}|z,\hat{x})]}_\text{conditional reconstruction loss} 
\end{split}
\end{equation}
where $\beta$ and $\gamma$ are hyperparameters. 
Both the end-to-end reconstruction loss and the conditional reconstruction loss use the same encoded latent space, but they use different conditional states, which are indicated with the hat notation. 
We use properties of the conditional entropy distribution and workspace coverage to dynamically tune the hyperparameters from Eq.~\ref{eq:cvae-loss-new} 
as listed in Line~\ref{line:hyperparams} of Alg.~\ref{alg:learning}. 
$\beta$ increases as the model discriminates objects, and $\gamma$ increases as more data is collected throughout the search space.
\begin{figure}[tb]      
    \centering
    \includegraphics[width=0.9\columnwidth]{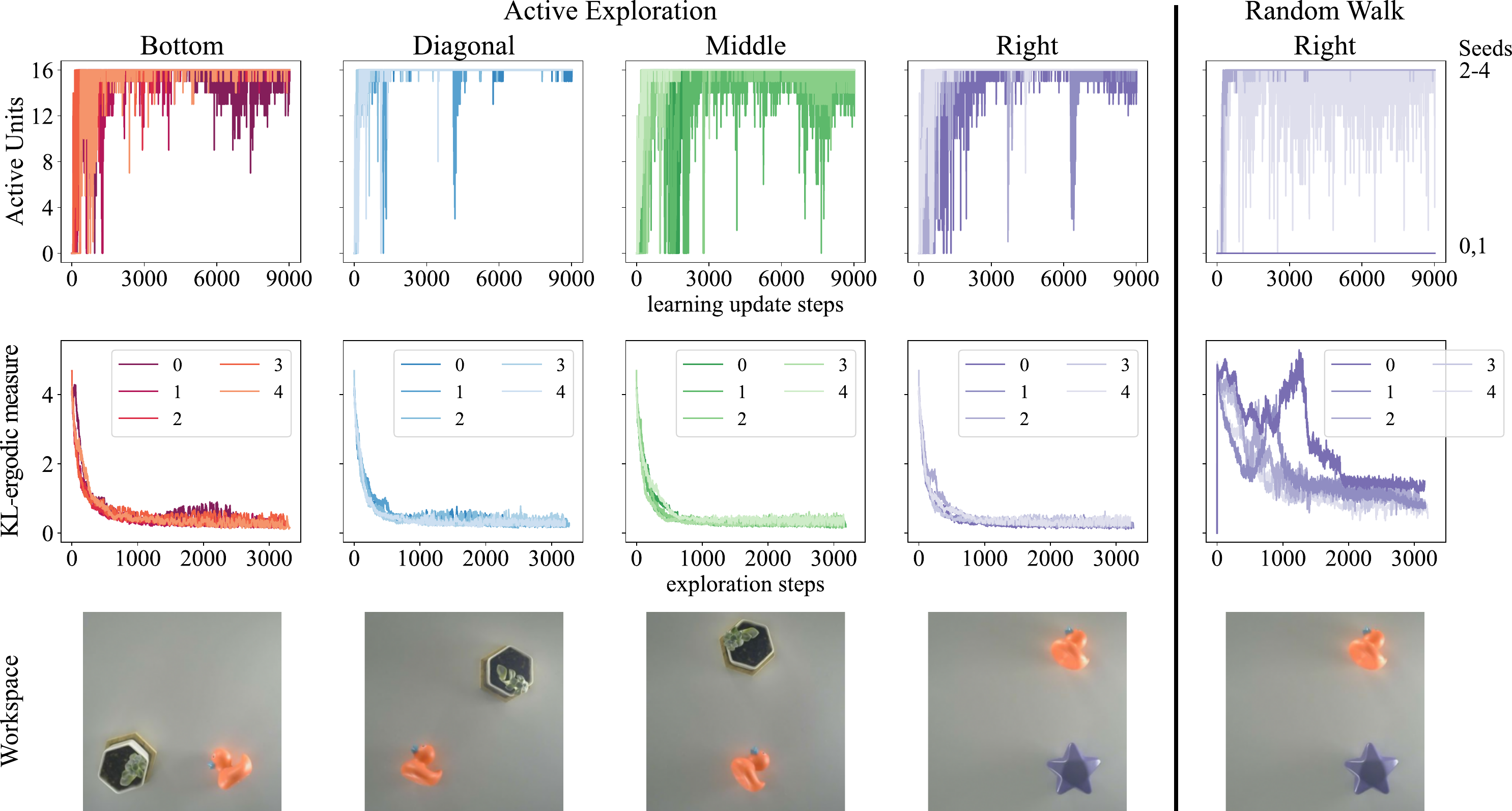}
    \caption{
    \textbf{Learning Metrics.} The first four columns show active exploration for different learning environments with 5 different seeds. 
    The last column shows data collected with a random walk.
    The top row shows the number of active units for each mini-batch update. Zero indicates latent space collapse, and 
    the higher the number of active units, the more the latent space is being used by the decoder. 
    The middle row shows the ergodic metric for each exploration step---representing how well the collected data matches the conditional entropy distribution. 
    An ergodic measure of zero would mean the collected data exactly matches the target distribution. 
    The bottom row shows a top-down view of each learning workspace.
    These plots show for all active seeds, the number of active units remained high and collected data quickly converged to the conditional entropy distribution. For the random walk, two seeds experienced latent space collapse and the data collected did not match the conditional entropy distribution. 
    }\label{fig:results} 
    \vspace{-1.5em}
\end{figure}
\subsubsection{Active Units} \vspace{-1em}
An aim of this work is to actively learn latent features of unknown objects. 
To quantify the latent space active units, we use the method from~\cite{truong2021bilateral}:
\begin{equation}
\text{AU} =  \sum\nolimits^L_{l=1} \mathbf{1}[ \text{Var}_{\mathcal{S}}({\mu}_{(z,l)})> \tau], \label{eq:active-units}
\end{equation}
where L is the size of the latent space, ${\mu}_z$ are the latent space means, $\mathcal{S}$ a subset of the collected data points, $\text{Var}(\cdot)$ is the variance across $\mathcal{S}$, $\mathbf{1}[\cdot]$ is an indicator function, and $\tau$ is the active unit threshold.
In Section~\ref{sec:vae}, we discussed latent space posterior collapse. When a latent space fully collapses, the latent space contains no active units. The larger the number of active units, the more dimensions of the latent space are being used by the decoder.

\subsubsection{Ablation} \vspace{-1em}
Although no ablation study is executed here, ablation plays an important role in benchmarking learning systems, typically in simulation. Hardware ablation studies should assess the energy expenditure of the algorithm, the rate at which the system learns, and the information content of the learned models over time. Future ablation studies could include replacing active exploration with alternate data collection methods, removing dynamic tuning of the loss hyperparameters ($\beta,\gamma$)\footnote{Hyperparameters could be set to fixed values or follow a pre-determined ramp.}, and testing curated sets of test objects\footnote{For example, testing ducks with different patterns and objects with different ratios of object size to camera field of view.} 

\subsubsection{Object Fingerprinting} \vspace{-1em}
Our method identifies learned objects in the environment by clustering samples from the conditional entropy distribution.\footnote{See the \hyperlink{appendix}{Appendix} for more details.} The robot collects representative data (sensor data, states, and latent spaces) for each identified learning object. We refer to each object dataset as a \emph{fingerprint}.

\subsubsection{Object Identification} \vspace{-1em}
During object identification, the learned CVAE model and object fingerprints are used as measurement models to test data collected from the new environment. For each identification test environment, a new model is learned from scratch using the active learning method from Alg.~\ref{alg:learning}. Data collected by the new model is processed by each learned measurement model and incorporated into belief grids for each learned object.

\section{Results}\label{sec:numerical-results}
\begin{wrapfigure}{r}{0.35\textwidth}
    \centering
    \vspace{-3.5em}
    \includegraphics[width=0.3\columnwidth]{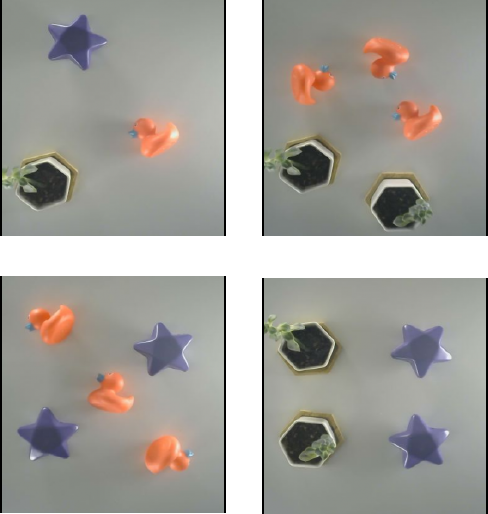} 
    \caption{ \centering
    \textbf{Test Workspaces}
    }\label{fig:id} 
    \vspace{-2.5em}
\end{wrapfigure}

The experiments are designed to highlight the ability of our method to actively learn information-rich latent spaces in an embodied framework.
Fig.~\ref{fig:results} shows the learning metrics for four different training workspaces, each containing two objects. For each workspaces, we test five different algorithmic seeds (20 tests total). All active models learn information-rich latent spaces and quickly collected data consistent with their conditional entropy distributions. The last column of Fig.~\ref{fig:results} shows data collected with a random walk, which collapses for 2/5 seeds. 

After training, we generate fingerprints for each actively learned object. We first test these fingerprints against their respective training environments. Then, we test these fingerprints in new workspaces with extra objects shown in Fig.~\ref{fig:id}. For comparison purposes, we collect 2000 test images for each identification test and test all 60 collected fingerprints on the same images.\footnote{We get 60 fingerprints from 4 test envs. $\times$ 5 seeds per env. $\times$ 3 fingerprints per env.} Therefore, the data quality of the identification process is solely the result of objects in the environment, not the belief about the presence of particular fingerprints.

Entries in Table~\ref{tab:results} summarize the number of objects tested across all seeds in the form $\sfrac{\text{success}}{\text{total tested}}$. For learning environments, the denominator of 5 denotes the 5 tested seeds. 
For test environments, the denominator includes product of the number of objects in the test environments and the number of stored fingerprints of that object type.\footnote{E.g. only one training env. has a star, so the denominator of the last entry in the first row is 5. Stars \& plants contains 2 stars, so the denominator~for star is 10.} 
Any entry with 0/0 indicates that object was not present in the identification environment. 
Blank spaces were omitted from the evaluation of the test environment due to number of objects present. 
The final column and final row contain total percentages by test environment and object respectively.
These results show that all fingerprints successfully 
\renewcommand{\tabcolsep}{0.3em}
\renewcommand{\arraystretch}{1.1} % General space between rows (1 standard)
\begin{table}[bth]
    \centering
    \caption{\textbf{Identification Results for 2000 Images}. (Left) Training results show all active learning fingerprints successfully identified all test objects. (Right) Test results show that over 75\% of the test object were successfully identified by the fingerprints.}\label{tab:results}
    \begin{tabular}{l|c|c|c|c||l|c|c|c|c} 
     & \multicolumn{4}{c||}{Objects} &  & \multicolumn{3}{c|}{Objects}& \itshape  Total\\
    Training & Duck & Plant & Star & Blank  & \multicolumn{1}{c|}{Test}   & Duck & Plant & Star  & \itshape  by Env.\\ \hline 
    Bottom & 5/5 & 5/5 & 0/0 & 5/5 & Star,Plant,Duck & 18/20 & 14/15 & 4/5 & \itshape  90\%\\
    Diagonal & 5/5 & 5/5 & 0/0  & 5/5 & Plants \& Ducks & 54/60 & 28/30 & 0/0  & \itshape  92\%\\
    Middle &  5/5 & 5/5 & 0/0 & 5/5 & Stars \& Ducks & 47/60 & 0/0 & 8/10 & \itshape  79\%\\
    Right & 5/5 & 0/0 & 5/5 & 5/5 &  Stars \& Plants & 0/0 & 24/30 & 9/10 & \itshape  83\% \\ \hline
    \hfill \itshape Total &  \itshape 100\% & \itshape 100\% & \itshape 100\% & \itshape 100\%  & \hfill \itshape  Total by Object & \itshape 85\% & \itshape 88\% & \itshape 85\% & 
    \end{tabular}
    \vspace{-1.5em}
\end{table}
identified the learned objects in their training environments and over $75\%$ of all fingerprints  successfully identified the learned objects the test environments. 
By allowing the identification process to actively collect data where the fingerprint beliefs are uncertain, this percentage could be increased. 

\section{Discussion and Conclusion}\label{sec:discussion}

Limitations of the method described here could be addressed in future work.
For active exploration, if the modeled dynamics substantially differ from the actual system, the robot will not be able to collect the planned data, and the model could collapse. Additionally, the presented results use a fixed number of exploration steps. Future work could implement real-time clustering to detect when the latent space has stabilized as a termination criteria. 
For fingerprinting, if objects are too close together, they could be clustered into a single object. To mitigate this limitation, the robot would need the ability to move in additional dimensions or with an interaction capability.
Another limitation is that fingerprint samples are currently selected with a normal distribution around the fingerprint center, so fingerprints could contain redundant data. Fingerprinting could be improved by implementing active coverage in the fingerprint generation phase. 
Finally, the current belief grids have a fixed number of points per conditional state dimension, so a different belief representation may be required for tests with more conditional variables. 

Our method here does not require pre-existing datasets, offline compute, or communication to the cloud to learn latent vision perception models. Instead, by exploiting control and the subsequent ability to shape the data it acquires, the robot regulates the learning process through data collection, reasoning about needed data based on the evolving state of the CVAE at every learning iterate. In contrast to next-best-view, our approach creates a sequence of high quality views at every time step, with subsequent guarantees on diversity of data (since no two views can be the same) and domain coverage. Moreover, our approach does not suffer from the catastrophic collapse of the latent space one gets when using randomized search across the domain. The method creates generative models of the domain, which can be used to separate features into distinct objects. 
Furthermore, the generative models can then be used to search for the objects. 
Lastly, as might be expected from a technique that uses continuous feedback from the learning model as it evolves, performance varies little from experiment to experiment---every data collection trial leads to generative models with information-rich latent spaces for search and identification. 
Experiments in physical hardware confirm predicted results.

{ \small
\noindent \textbf{Acknowledgements} {This material is supported by ONR Grant N00014-21-1-2706 and Army Research Office Grant W911NF-22-1-0286. Any opinions, findings and conclusions or recommendations expressed in this material are those of the authors and do not necessarily reflect the views of this institution.}  \vspace{-1em} } 

% ---- Bibliography ----

\bibliographystyle{splncs03}
\bibliography{references}

\section*{APPENDIX}\hypertarget{appendix} 
Robot communication was handled by a laptop (Intel\textregistered~Core i7-8550U CPU @ 1.80GHz x 8) running Ubuntu 18.04, Realtime Kernel (version 5.4.138-rt62), and ROS Melodic~\cite{Quigley09ros}. All other processes were run on a separate server (Intel\textregistered~Xeon(R) Platinum 8380 CPU @ 2.30GHz x 160) with Ubuntu 20.04 and ROS Noetic. Computers were connected via Ethernet on a local network. 

\renewcommand{\tabcolsep}{0.3em}
\renewcommand{\arraystretch}{1.1} % General space between rows (1 standard)
\begin{table}[tbh]
    \centering
    \scriptsize
    \caption{CVAE Network Architecture}
    \begin{tabular}{|lr|l|} \hline
    {\itshape Key }
    & CNN & 
    \hfill \itshape 2D Conv. Filters @ Kernel (stride, pad) $+$ Activation  \\
    & MLP &
    \hfill \itshape Fully Connected Output $+$ Activation  \\
     \multicolumn{2}{|r|}{ Latent Space Samples} & \hfill $z_s = \mu_z + \sigma_z\epsilon $ where  $ \epsilon \sim \mathcal{N}(\mathbf{0},\mathbf{I}) $ and $ z = \mathcal{N}(\mu_z,\text{diag}(\sigma_z) $ \\
    \hline \hline 
    {\bfseries Sensor} 
    & Input & 3x180x180 \hfill \itshape $\triangleright$ RGB Image    \\
    {\bfseries Encoder} 
    & CNN &
        10@3x3(2,0) + ReLU $\rightarrow$ % \\ & &
        10@3x3(2,0) + ReLU $\rightarrow$ 
        20@5x5(3,0) 
    \\  
    & Output & 20x14x14 \hfill \itshape $\triangleright$ Encoded Image \\
    \hline
    {\bfseries Main } 
    & Input & 3923 \hfill \itshape $\triangleright$ Flattened Encoded Image $+$ State $(x,y,\theta)$   \\
    {\bfseries Encoder } 
    & MLP & 
        512 + ReLU  $\rightarrow$ 
        265 + ReLU $\rightarrow$  
        32 
    \\ 
    & Output & 32 \hfill \itshape $\triangleright$ Latent Space $ (z) $ \\ \hline \hline
    {\bfseries Main  } 
    & Input & 19 \hfill \itshape $\triangleright$ Latent Space Samples $+$ State $(x,y,\theta)$ \\
    {\bfseries Decoder  } 
    & MLP &  
        256 + ReLU  $\rightarrow$ 
        512 + ReLU $\rightarrow$ 
        3921 
    \\ 
    & Output & 3921  \hfill \itshape $\triangleright$ Sensor Prediction $+$ Variance   \\ \hline
    {\bfseries Sensor } 
    & Input & 20x14x14  \hfill \itshape $\triangleright$ UnFlattened Decoder Sensor Prediction \\ 
    {\bfseries Decoder} 
    & CNN &
        20@5x5(3,0) + ReLU $\rightarrow$
        10@3x3(2,0) + ReLU $\rightarrow$
        10@3x3(2,1) 
    \\  
    & Output & 3x180x180 \hfill \itshape $\triangleright$ RGB image \\ \hline
    \end{tabular}\label{tab:cvae-params}
    \vspace{-1.5em}
\end{table}

\begin{wrapfigure}{r}{0.38\textwidth}
    \centering
    \vspace{-0.75em}
    \includegraphics[width=0.36\columnwidth]{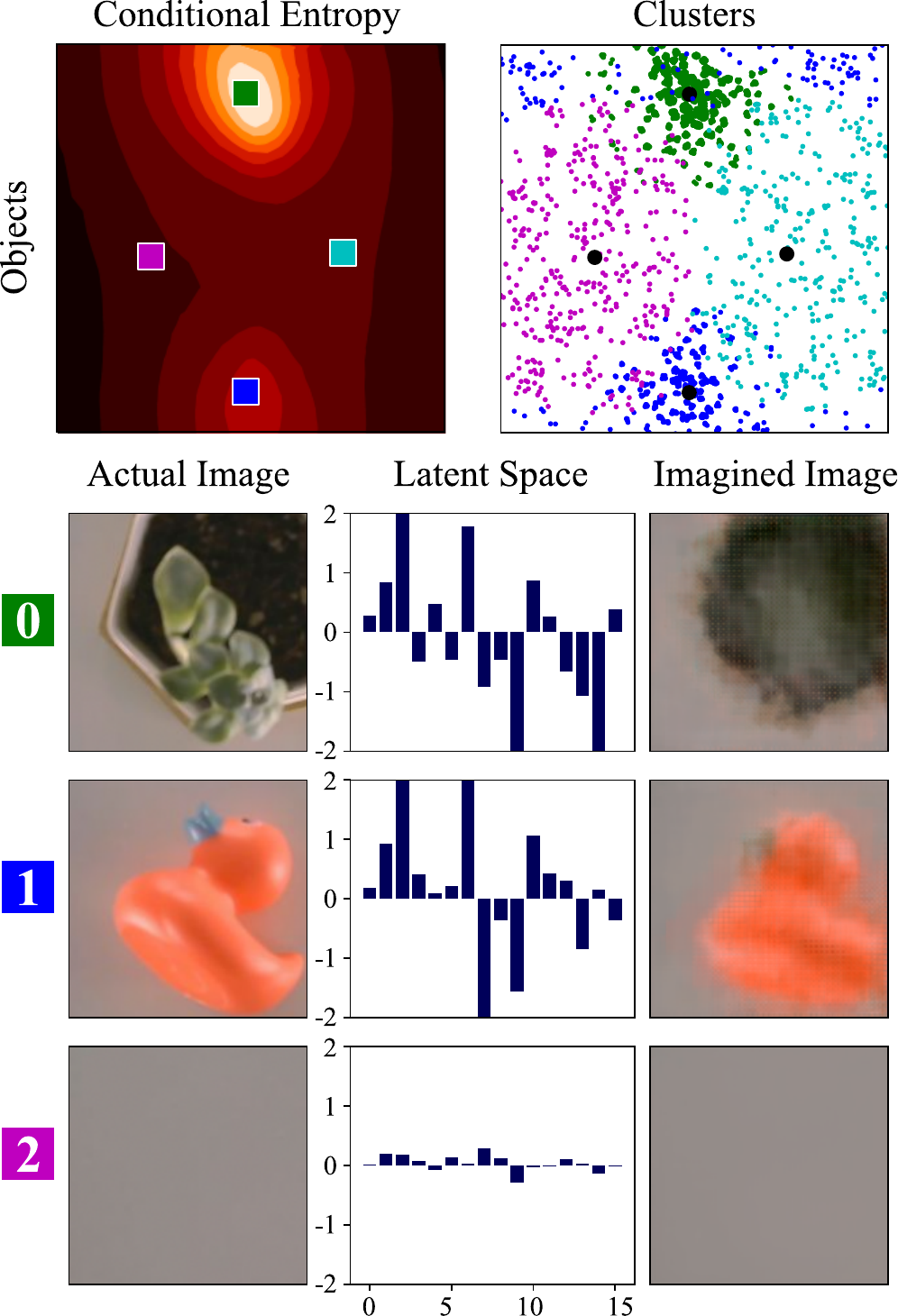} 
    \vspace{-.5em}
    \caption{\fontsize{8}{10}\selectfont % footnotesize overwritten in class
    \textbf{Object Clustering.} (Top Left) conditional entropy distribution, squares are object locations; (Top Right) clustered conditional entropy samples; colors match the objects in the left plot. (Bottom Rows) data collected for 3 objects. 
    }\label{fig:clustering} 
    \vspace{-3.5em}
\end{wrapfigure}
\subsubsection{Models} \vspace{-1.5em}
All networks in  Table~\ref{tab:cvae-params} are implemented in PyTorch v2.0~\cite{pytorch_NEURIPS2019_9015}. We use the Adam~\cite{kingma2014adam} optimizer with a learning rate of $0.001$, a batch size of $64$, and $4$ distributed trainers. Distributed training (optional) allows asynchronous training updates and robot control updates. Training frequency is throttled to 3 model updates per data collection step. 
For tuning $\beta$, we use $\chi = 4$ (i.e. $\min\beta=10^{-4}$).

\subsubsection{Control} \vspace{-1.25em}
End-effector dynamics are modeled as a double integrator. We assume states are controlled independently and use barrier functions to enforce workspace boundaries. 
We use a $10$ step planning horizon and a time step of $0.2$s. 

\subsubsection{Object Fingerprinting} \vspace{-1.25em}
We use importance sampling to generate conditional entropy and use Mean Shift Clustering to find object locations~\cite{comaniciu2002meanshift}. ``Blank'' regions are valid objects. For each object, the robot generates a \emph{fingerprint}---a table of images, states, and latent space samples. We store 50 images per object. Fig.~\ref{fig:clustering} shows an object clustering example. 

\subsubsection{Object Identification} \vspace{-1.25em}
During identification, regularly spaced belief grids are evaluated for each fingerprint. Each grid-point has a mean and a standard deviation representing how likely it is that a particular object is present the test workspace. The following are completed for each test image: 1) the stored CVAE encoder generates latent space samples for collected image paired with each stored fingerprint state; 2) the L2 distance between stored latent space and each new latent space is calculated; 3) a likelihood grid is generated using the transform between test environment and most likely measurement; 4) the belief grid is updated using Bayes rule. For these tests, we use the minimum distance between fingerprints as the distance threshold for the measurement model. 

\end{document}